\batchmode
\makeatletter
\def\input@path{{F:/important/doingWork/myWorks/AAAI2015//}}
\makeatother
\documentclass[a4paper]{article}
\usepackage{helvet}
\usepackage{courier}
\usepackage[utf8]{inputenc}
\usepackage{color}
\definecolor{page_backgroundcolor}{rgb}{1, 1, 1}
\pagecolor{page_backgroundcolor}
\usepackage{float}
\usepackage{amsmath}
\usepackage{amssymb}
\usepackage{graphicx}
\usepackage[unicode=true,
 bookmarks=false,
 breaklinks=false,pdfborder={0 0 0},backref=false,colorlinks=true]
 {hyperref}
\hypersetup{
 citecolor=blue, urlcolor=blue}

\makeatletter

\pdfpageheight\paperheight
\pdfpagewidth\paperwidth

\providecommand{\tabularnewline}{\\}
\floatstyle{ruled}
\newfloat{algorithm}{tbp}{loa}
\providecommand{\algorithmname}{Algorithm}
\floatname{algorithm}{\protect\algorithmname}

\usepackage{aaai}
\usepackage{times}
\usepackage{helvet}
\usepackage{courier}
\usepackage{breakurl}
\usepackage{epsfig}
\usepackage{multirow}
\usepackage{mdwlist}
\usepackage{algorithm}
\usepackage{algorithmic}% 
\usepackage{graphicx} 
\usepackage{url}
\usepackage{array}
\usepackage{float}
\usepackage{multirow}

\usepackage{amssymb}
\usepackage{amsthm}
\newtheorem{theorem}{Theorem}  % define  "Theorem"
\newtheorem{corollary}[theorem]{Corollary}     % define  "Definition"
\newtheorem{remark}[theorem]{Remark}     % define  "Definition"
\makeatother

\begin{document}
\global\long\def\mtbfA{\mathbf{A}}
 \global\long\def\mtbfa{\mathbf{a}}
 \global\long\def\mebfA{\bar{\mtbfA}}
 \global\long\def\mebfa{\bar{\mtbfa}}

\global\long\def\mhbfA{\widehat{\mathbf{A}}}
 \global\long\def\mhbfa{\widehat{\mathbf{a}}}
 \global\long\def\mtcalA{\mathcal{A}}

\global\long\def\mtbfB{\mathbf{B}}
 \global\long\def\mtbfb{\mathbf{b}}
 \global\long\def\mebfB{\bar{\mtbfB}}
 \global\long\def\mebfb{\bar{\mtbfb}}

\global\long\def\mhbfB{\widehat{\mathbf{B}}}
 \global\long\def\mhbfb{\widehat{\mathbf{b}}}
 \global\long\def\mtcalB{\mathcal{B}}

\global\long\def\mtbfC{\mathbf{C}}
 \global\long\def\mtbfc{\mathbf{c}}
 \global\long\def\mebfC{\bar{\mtbfC}}
 \global\long\def\mebfc{\bar{\mtbfc}}

\global\long\def\mhbfC{\widehat{\mathbf{C}}}
 \global\long\def\mhbfc{\widehat{\mathbf{c}}}
 \global\long\def\mtcalC{\mathcal{C}}
 \global\long\def\mtbbC{\mathbb{C}}

\global\long\def\mtbfD{\mathbf{D}}
 \global\long\def\mtbfd{\mathbf{d}}
 \global\long\def\mebfD{\bar{\mtbfD}}
 \global\long\def\mebfd{\bar{\mtbfd}}

\global\long\def\mhbfD{\widehat{\mathbf{D}}}
 \global\long\def\mhbfd{\widehat{\mathbf{d}}}
 \global\long\def\mtcalD{\mathcal{D}}

\global\long\def\mtbfE{\mathbf{E}}
 \global\long\def\mtbfe{\mathbf{e}}
 \global\long\def\mebfE{\bar{\mtbfE}}
 \global\long\def\mebfe{\bar{\mtbfe}}

\global\long\def\mhbfE{\widehat{\mathbf{E}}}
 \global\long\def\mhbfe{\widehat{\mathbf{e}}}
 \global\long\def\mtcalE{\mathcal{E}}
 \global\long\def\mtexpect{\mathbb{E}}

\global\long\def\mtbfF{\mathbf{F}}
 \global\long\def\mtbff{\mathbf{f}}
 \global\long\def\mebfF{\bar{\mathbf{F}}}
 \global\long\def\mebff{\bar{\mathbf{f}}}

\global\long\def\mhbfF{\widehat{\mathbf{F}}}
 \global\long\def\mhbff{\widehat{\mathbf{f}}}
 \global\long\def\mtcalF{\mathcal{F}}

\global\long\def\mtbfG{\mathbf{G}}
 \global\long\def\mtbfg{\mathbf{g}}
 \global\long\def\mebfG{\bar{\mathbf{G}}}
 \global\long\def\mebfg{\bar{\mathbf{g}}}

\global\long\def\mhbfG{\widehat{\mathbf{G}}}
 \global\long\def\mhbfg{\widehat{\mathbf{g}}}
 \global\long\def\mtcalG{\mathcal{G}}

\global\long\def\mtbfH{\mathbf{H}}
 \global\long\def\mtbfh{\mathbf{h}}
 \global\long\def\mebfH{\bar{\mathbf{H}}}
 \global\long\def\mebfh{\bar{\mathbf{h}}}

\global\long\def\mhbfH{\widehat{\mathbf{H}}}
 \global\long\def\mhbfh{\widehat{\mathbf{h}}}
 \global\long\def\mtcalH{\mathcal{H}}

\global\long\def\mtbfI{\mathbf{I}}
 \global\long\def\mtbfi{\mathbf{i}}
 \global\long\def\mebfI{\bar{\mathbf{I}}}
 \global\long\def\mebfi{\bar{\mathbf{i}}}

\global\long\def\mhbfI{\widehat{\mathbf{I}}}
 \global\long\def\mhbfi{\widehat{\mathbf{i}}}
 \global\long\def\mtcalI{\mathcal{I}}

\global\long\def\mtbfJ{\mathbf{J}}
 \global\long\def\mtbfj{\mathbf{j}}
 \global\long\def\mebfJ{\bar{\mathbf{J}}}
 \global\long\def\mebfj{\bar{\mathbf{j}}}

\global\long\def\mhbfJ{\widehat{\mathbf{J}}}
 \global\long\def\mhbfj{\widehat{\mathbf{j}}}
 \global\long\def\mtcalJ{\mathcal{J}}

\global\long\def\mtbfK{\mathbf{K}}
 \global\long\def\mtbfk{\mathbf{k}}
 \global\long\def\mebfK{\bar{\mathbf{K}}}
 \global\long\def\mebfk{\bar{\mathbf{k}}}

\global\long\def\mhbfK{\widehat{\mathbf{K}}}
 \global\long\def\mhbfk{\widehat{\mathbf{k}}}
 \global\long\def\mtcalK{\mathcal{K}}

\global\long\def\mtbfL{\mathbf{L}}
 \global\long\def\mtbfl{\mathbf{l}}
 \global\long\def\mebfL{\bar{\mathbf{L}}}
 \global\long\def\mebfl{\bar{\mathbf{l}}}

\global\long\def\mhbfL{\widehat{\mathbf{K}}}
 \global\long\def\mhbfl{\widehat{\mathbf{k}}}
 \global\long\def\mtcalL{\mathcal{L}}

\global\long\def\mtbfM{\mathbf{M}}
 \global\long\def\mtbfm{\mathbf{m}}
 \global\long\def\mebfM{\bar{\mathbf{M}}}
 \global\long\def\mebfm{\bar{\mathbf{m}}}

\global\long\def\mhbfM{\widehat{\mathbf{M}}}
 \global\long\def\mhbfm{\widehat{\mathbf{m}}}
 \global\long\def\mtcalM{\mathcal{M}}

\global\long\def\mtbfN{\mathbf{N}}
 \global\long\def\mtbfn{\mathbf{n}}
 \global\long\def\mebfN{\bar{\mathbf{N}}}
 \global\long\def\mebfn{\bar{\mathbf{n}}}

\global\long\def\mhbfN{\widehat{\mathbf{N}}}
 \global\long\def\mhbfn{\widehat{\mathbf{n}}}
 \global\long\def\mtcalN{\mathcal{N}}

\global\long\def\mtbfO{\mathbf{O}}
 \global\long\def\mtbfo{\mathbf{o}}
 \global\long\def\mebfO{\bar{\mathbf{O}}}
 \global\long\def\mebfo{\bar{\mathbf{o}}}

\global\long\def\mhbfO{\widehat{\mathbf{O}}}
 \global\long\def\mhbfo{\widehat{\mathbf{o}}}
 \global\long\def\mtcalO{\mathcal{O}}

\global\long\def\mtbfP{\mathbf{P}}
 \global\long\def\mtbfp{\mathbf{p}}
 \global\long\def\mebfP{\bar{\mathbf{P}}}
 \global\long\def\mebfp{\bar{\mathbf{p}}}

\global\long\def\mhbfP{\widehat{\mathbf{P}}}
 \global\long\def\mhbfp{\widehat{\mathbf{p}}}
 \global\long\def\mtcalP{\mathcal{P}}

\global\long\def\mtbfQ{\mathbf{Q}}
 \global\long\def\mtbfq{\mathbf{q}}
 \global\long\def\mebfQ{\bar{\mathbf{Q}}}
 \global\long\def\mebfq{\bar{\mathbf{q}}}

\global\long\def\mhbfQ{\widehat{\mathbf{Q}}}
 \global\long\def\mhbfq{\widehat{\mathbf{q}}}
\global\long\def\mtcalQ{\mathcal{Q}}

\global\long\def\mtbfR{\mathbf{R}}
 \global\long\def\mtbfr{\mathbf{r}}
 \global\long\def\mebfR{\bar{\mathbf{R}}}
 \global\long\def\mebfr{\bar{\mathbf{r}}}

\global\long\def\mhbfR{\widehat{\mathbf{R}}}
 \global\long\def\mhbfr{\widehat{\mathbf{r}}}
\global\long\def\mtcalR{\mathcal{R}}
 \global\long\def\mtbbR{\mathbb{R}}

\global\long\def\mtbfS{\mathbf{S}}
 \global\long\def\mtbfs{\mathbf{s}}
 \global\long\def\mebfS{\bar{\mathbf{S}}}
 \global\long\def\mebfs{\bar{\mathbf{s}}}

\global\long\def\mhbfS{\widehat{\mathbf{S}}}
 \global\long\def\mhbfs{\widehat{\mathbf{s}}}
\global\long\def\mtcalS{\mathcal{S}}

\global\long\def\mtbfT{\mathbf{T}}
 \global\long\def\mtbft{\mathbf{t}}
 \global\long\def\mebfT{\bar{\mathbf{T}}}
 \global\long\def\mebft{\bar{\mathbf{t}}}

\global\long\def\mhbfT{\widehat{\mathbf{T}}}
 \global\long\def\mhbft{\widehat{\mathbf{t}}}
 \global\long\def\mtcalT{\mathcal{T}}

\global\long\def\mtbfU{\mathbf{U}}
 \global\long\def\mtbfu{\mathbf{u}}
 \global\long\def\mebfU{\bar{\mathbf{U}}}
 \global\long\def\mebfu{\bar{\mathbf{u}}}

\global\long\def\mhbfU{\widehat{\mathbf{U}}}
 \global\long\def\mhbfu{\widehat{\mathbf{u}}}
 \global\long\def\mtcalU{\mathcal{U}}

\global\long\def\mtbfV{\mathbf{V}}
 \global\long\def\mtbfv{\mathbf{v}}
 \global\long\def\mebfV{\bar{\mathbf{V}}}
 \global\long\def\mebfv{\bar{\mathbf{v}}}

\global\long\def\mhbfV{\widehat{\mathbf{V}}}
 \global\long\def\mhbfv{\widehat{\mathbf{v}}}
\global\long\def\mtcalV{\mathcal{V}}

\global\long\def\mtbfW{\mathbf{W}}
 \global\long\def\mtbfw{\mathbf{w}}
 \global\long\def\mebfW{\bar{\mathbf{W}}}
 \global\long\def\mebfw{\bar{\mathbf{w}}}

\global\long\def\mhbfW{\widehat{\mathbf{W}}}
 \global\long\def\mhbfw{\widehat{\mathbf{w}}}
 \global\long\def\mtcalW{\mathcal{W}}

\global\long\def\mtbfX{\mathbf{X}}
 \global\long\def\mtbfx{\mathbf{x}}
 \global\long\def\mebfX{\bar{\mtbfX}}
 \global\long\def\mebfx{\bar{\mtbfx}}

\global\long\def\mhbfX{\widehat{\mathbf{X}}}
 \global\long\def\mhbfx{\widehat{\mathbf{x}}}
 \global\long\def\mtcalX{\mathcal{X}}

\global\long\def\mtbfY{\mathbf{Y}}
 \global\long\def\mtbfy{\mathbf{y}}
\global\long\def\mebfY{\bar{\mathbf{Y}}}
 \global\long\def\mebfy{\bar{\mathbf{y}}}

\global\long\def\mhbfY{\widehat{\mathbf{Y}}}
 \global\long\def\mhbfy{\widehat{\mathbf{y}}}
 \global\long\def\mtcalY{\mathcal{Y}}

\global\long\def\mtbfZ{\mathbf{Z}}
 \global\long\def\mtbfz{\mathbf{z}}
 \global\long\def\mebfZ{\bar{\mathbf{Z}}}
 \global\long\def\mebfz{\bar{\mathbf{z}}}

\global\long\def\mhbfZ{\widehat{\mathbf{Z}}}
 \global\long\def\mhbfz{\widehat{\mathbf{z}}}
\global\long\def\mtcalZ{\mathcal{Z}}

\global\long\def\mtvar{\mathbf{\text{Var}}}

\global\long\def\mtth{\text{th}}

\global\long\def\mtbfzero{\mathbf{0}}
 \global\long\def\mtbfone{\mathbf{1}}

\global\long\def\mttrace{\text{Tr}}

\global\long\def\mttotalVariation{\text{TV}}

\global\long\def\mtdet{\text{Det}}

\global\long\def\mtvec{\mathbf{\text{vec}}}

\global\long\def\mtvar{\mathbf{\text{var}}}

\global\long\def\mtcov{\mathbf{\text{cov}}}

\global\long\def\mtsubTo{\mathbf{\text{s.t.}}}

\global\long\def\mtfor{\text{for}}

\global\long\def\mtrank{\text{rank}}

\global\long\def\mtdiag{\mathbf{\text{diag}}}

\global\long\def\mtsign{\mathbf{\text{sign}}}

\global\long\def\mtloss{\mathbf{\text{loss}}}

\global\long\def\mtwhen{\text{when}}

\global\long\def\mtexpect{\mathbb{E}}

\global\long\def\mtcalN{\mathcal{N}}

\global\long\def\mtbbR{\mathbb{R}}

\title{10,000+ Times Accelerated Robust Subset Selection (ARSS)}

\author{{\normalsize{Feiyun Zhu, Bin Fan, Xinliang Zhu, Ying Wang, Shiming
Xiang and Chunhong Pan}}\\
Institute of Automation, Chinese Academy of Sciences\\
\{fyzhu, bfan, ywang, smxiang and chpan\}@nlpr.ia.ac.cn, zhuxinliang2012@ia.ac.cn}
\maketitle
\begin{abstract}
Subset selection from massive data with noised information is increasingly
popular for various applications. This problem is still highly challenging
as current methods are generally slow in speed and sensitive to outliers.
To address the above two issues, we propose an accelerated robust
subset selection (ARSS) method. Specifically in the subset selection
area, this is the first attempt to employ the $\ell_{p}\left(0<p\leq1\right)$-norm
based measure for the representation loss, preventing large errors
from dominating our objective. As a result, the robustness against
outlier elements is greatly enhanced. Actually, data size is generally
much larger than feature length, i.e. $N\!\gg\! L$. Based on this
observation, we propose a speedup solver (via ALM and equivalent derivations)
to\textbf{ }highly reduce the computational cost, theoretically from
$O\left(N^{4}\right)$ to $O\left(N{}^{2}L\right)$. Extensive experiments
on ten benchmark datasets verify that our method not only outperforms
state of the art methods, but also runs 10,000+ times faster than
the most related method.
\end{abstract}

\section{Introduction}

Due to the explosive growth of data\ \cite{junWang_2012_PAMI_SemiSupHash},
subset selection methods are increasingly popular for a wide range
of machine learning and computer vision applications\ \cite{Frey_2007_Science_ClusteringViaPassingMessages,Jenatton_2011_JMLR_SelectionViaSparsityInducingNorms}.
This kind of methods offer the potential to select a few highly representative
samples or exemplars to describe the entire dataset. By analyzing
a few, we can roughly know all. Such case is very important to summarize
and visualize huge datasets of texts, images and videos etc.\ \cite{Bien_2012_AAS_prototypeSelection_4_Interpret,Elhamifar_2012_NIPS_ExemplarsbyPairwise}.
Besides, by only using the selected exemplars for succeeding tasks,
the cost of memories and computational time will be greatly reduced\ \cite{Garcia_2012_PAMI_PrototypeSelection_4_KNN}.
Additionally, as outliers are generally less representative, the side
effect of outliers will be reduced, thus boosting the performance
of subsequent applications\ \cite{Elhamifar_2012_CVPR_SeeAllLookFew}.
\begin{figure}[t]
\centering{}\includegraphics[width=1\columnwidth]{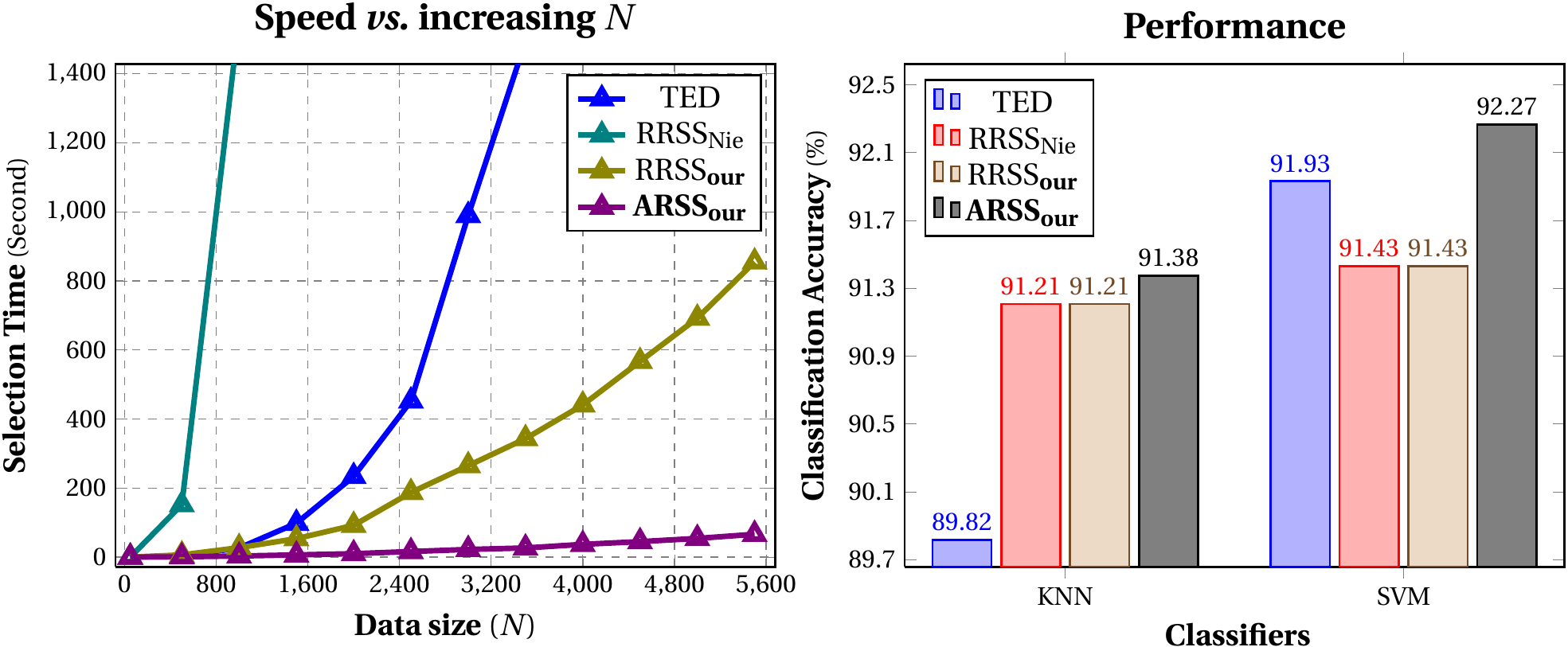}
\caption{{\footnotesize{Comparisons of four algorithms on Optdigit. Two conclusions
can be drawn. First, our method (ARSS$_{\text{our}}$) is highly faster
than all others; with the help of an elegant new theorem, RRSS$_{\text{our}}$
is significantly faster than the authorial algorithm RRSS$_{\text{Nie}}$.
Second, ARSS$_{\text{our}}$ achieves highly promising prediction
accuracies. }}{\small{\label{fig:compare_4_methods_Optdigit}}}}
\end{figure}

There have been several subset selection methods. The most intuitional
method is to randomly select a fixed number of samples. Although highly
efficient, there is no guarantee for an effective selection. For the
other methods, depending on the mechanism of representative exemplars,
there are mainly three categories of selection methods. One category
relies on the assumption that the data points lie in one or multiple
low-dimensional subspaces. Specifically, the Rank Revealing QR (RRQR)\ \cite{Chan_1987_LAIA_RRQR,Boutsidis_2009_SODA_RRQR_2}
selects the subsets that give the best conditional sub-matrix. Unfortunately,
this method has suboptimal properties, as it is not assured to find
the globally optimum in polynomial time. 

Another category assumes that the samples are distributed around centers\ \cite{Frey_2007_Science_ClusteringViaPassingMessages,zcLin_2010_ICML_RobstSubspaceSegmentation}.
The center or its nearest neighbour are selected as exemplars. Perhaps,
Kmeans and Kmedoids are the most typical methods (Kmedoids is a variant
of Kmeans). Both methods employ an EM-like algorithm. Thus, the results
depend tightly on the initialization, and they are highly unstable
for large $K$ (i.e. the number of centers or selected samples). 

Recently, there are a few methods that assume exemplars are the samples
that can best represent the whole dataset. However, for\ \cite{kaiYu_2006_ICML_activeLearning},
the optimization is a combinatorial problem (NP-hard)\ \cite{FpNie_2013_IJCAI_EarlyActiveLearning,kaiYu_2008_SIGIR_NonGreedy_TED},
which is computationally intractable to solve. Besides, the representation
loss is measured by the least square measure, which is sensitive to
outliers in data\ \cite{huaWang_2014_ICML_RobustMetricLearning,fyzhu_2014_JSTSP_RRLbS,FpNie_2013_IJCAI_EarlyActiveLearning}.

Then\ \cite{FpNie_2013_IJCAI_EarlyActiveLearning} improves\ \cite{kaiYu_2006_ICML_activeLearning}
by employing a robust loss via the $\ell_{2,1}$-norm; the $\ell_{1}$-norm
is applied to samples, and the $\ell_{2}$-norm is used for features.
In this way, the side effect of outlier samples is relieved. The solver
of\ \cite{FpNie_2013_IJCAI_EarlyActiveLearning} is theoretically
perfect due to its ability of convergence to global optima. Unfortunately,
in terms of computational costs, the solver is highly complex. It
takes $O\left(N{}^{4}\right)$ for one iteration as shown in Table\ \ref{tab:complexity_3_methods}.
This is infeasible for the case of large $N$ (e.g. it takes 2000+
hours for a case of $N=13000$). Moreover, the representation loss
is only robust against outlier samples. Such case is worth improvement,
as there may exist outlier elements in real data.

\paragraph{Contributions. }

In this paper, we propose an accelerated robust subset selection method
to highly raise the speed on the one hand, and to boost the robustness
on the other. To this end, we use the $\ell_{p}\left(0\!<\! p\!\leq\!1\right)$-norm
based robust measure for the representation loss, preventing large
errors from dominating our objective. As a result, the robustness
against outliers is greatly boosted. Then, based on the observation
that data size is generally much larger than feature length, i.e.
$N\!\gg\! L$, we propose a speedup solver. The main acceleration
is owing to the Augmented Lagrange Multiplier (ALM) and an equivalent
derivation. Via them, we reduce the computational complexity from
$O\left(N{}^{4}\right)$ to $O\left(N{}^{2}L\right)$. Extensive results
on ten benchmark datasets demonstrate that in average, our method
is 10,000+ times faster than Nie's method. The selection quality is
highly encouraging as shown in Fig.\ \ref{fig:compare_4_methods_Optdigit}.
Additionally, via another equivalent derivation, we give an accelerated
solver for Nie's method, theoretically reducing the computational
complexity from $O\left(N{}^{4}\right)$ to $O\left(N{}^{2}L+NL{}^{3}\right)$
as listed in Table\,\ref{tab:complexity_3_methods}, empirically
obtaining a 500+ times speedup compared with the authorial solver.

\paragraph{Notations. }

We use boldface uppercase letters to denote matrices and boldface
lowercase letters to represent vectors. For a matrix $\mtbfY=\left[Y_{ln}\right]\in\mtbbR^{L\times N}$,
we denote its $l^{\mtth}$ row and $n^{\mtth}$ column as $\mtbfy^{l}$
and $\mtbfy_{n}$ respectively. The $\ell_{2,1}$-norm of a matrix
is defined as $\left\Vert \mtbfY\right\Vert _{2,1}\!=\!\sum_{l}^{L}\sqrt{\sum_{n}^{N}Y_{ln}^{2}}=\sum_{l}^{L}\left\Vert \mtbfy^{l}\right\Vert _{2}$.
The $\ell_{p}\left(0<p\leq1\right)$-norm of a matrix is defined as
$\left\Vert \mtbfY\right\Vert _{p}\!=\!\left(\sum_{n}^{N}\sum_{l}^{L}\left|Y_{ln}\right|^{p}\right)^{\frac{1}{p}}$;
thus, we have $\left\Vert \mtbfY\right\Vert _{p}^{p}=\sum_{l,n}\left|Y_{ln}\right|^{p}$.

\section{Subset Selection via Self-Representation}

In the problem of subset selection, we are often given a set of $N$
unlabelled points $\mtbfX\!=\!\left\{ \mtbfx_{1},\mtbfx_{2},\!\cdots\!,\mtbfx_{N}\!\mid\!\mtbfx_{n}\!\in\!\mtbbR^{L}\right\} $,
where $L$ is the feature length. The goal is to select the top $K\left(K\!\ll\! N\right)$
most representative and informative samples (i.e. exemplars) to effectively
describe the entire dataset $\mtbfX$. By solely using these $K$
exemplars for subsequent tasks, we could greatly reduce the computational
costs and largely alleviate the side effects of outlier elements in
data. Such a motivation could be formulated as the Transductive Experimental
Design (TED) model\,\cite{kaiYu_2006_ICML_activeLearning}: 
\begin{align}
\min_{\mtbfQ,\mtbfA} & \sum_{n=1}^{N}\left(\left\Vert \mtbfx_{n}-\mtbfQ\mtbfa_{n}\right\Vert _{2}^{2}+\alpha\left\Vert \mtbfa_{n}\right\Vert _{2}^{2}\right),\label{eq:TED}
\end{align}
where $\mtbfQ\in\mtbbR^{L\times K}$ is the selected subset matrix,
whose column vectors all come from $\mtbfX$, i.e. $\mtbfq_{k}\in\mtbfX,\forall k\in\left\{ 1,\cdots,K\right\} $;
$\mtbfA\!=\!\left[\mtbfa_{1},\cdots,\mtbfa_{N}\right]\in\mtbbR^{K\times N}$
is the corresponding linear combination coefficients. By minimizing\ \eqref{eq:TED},
TED could select the highly informative and representative samples,
as they have to well represent all the samples in $\mtbfX$. 

Although TED\ \eqref{eq:TED} is well modeled---very accurate and
intuitive, there are two bottlenecks. First, the objective is a combinatorial
optimization problem. It is NP-hard to exhaustively search the optimal
subset $\mtbfQ$ from $\mtbfX$. For this reason, the author approximate\ \eqref{eq:TED}
via a sequential optimization problem, which is solved by an inefficient
greedy optimization algorithm. Second, similar to the existing least
square loss based models in machine learning and statistics,\ \eqref{eq:TED}
is sensitive to the presence of outliers\ \cite{huaWang_2014_ICML_RobustMetricLearning}.
\begin{table}[t]
\centering{}\caption{{\footnotesize{}}{\small{Complexity comparison of three algorithms
at one iteration step. Generally, data size is much larger than feature
length, i.e. $N\!\gg\! L$. Compared with RRSS$_{\text{Nie}}$ (}}\emph{\small{Nie's
model via the authorial solver}}{\small{), RRSS$_{\text{our}}$ (}}\emph{\small{Nie's
method speeded up by our solver}}{\small{) and ARSS$_{\text{our}}$
(}}\emph{\small{ours}}{\small{) are significantly simplified.\label{tab:complexity_3_methods}}}}
\begin{tabular}{|c|c|c|c|}
\hline 
Methods &
RRSS$_{\text{Nie}}$ &
RRSS$_{\text{our}}$ &
ARSS$_{\text{our}}$\tabularnewline
\hline 
Complex. &
$O\left(N^{4}\right)$ &
$O\left(N^{2}L+NL^{3}\right)$ &
$O\left(N^{2}L\right)$\tabularnewline
\hline 
\end{tabular}
\end{table}

Accordingly, Nie \emph{et al.} propose a new model (RRSS):
\begin{equation}
\min_{\mtbfA\in\mtbbR^{N\times N}}\sum_{n=1}^{N}\left\Vert \mtbfx_{n}-\mtbfX\mtbfa_{n}\right\Vert _{2}+\gamma\left\Vert \mtbfA\right\Vert _{2,1},\label{eq:nieRRSS_vector}
\end{equation}
where $\gamma$ is a nonnegative parameter; $\mtbfA$ is constrained
to be row-sparse, and thus to select the most representative and informative
samples\ \cite{FpNie_2013_IJCAI_EarlyActiveLearning}. As the representation
loss is accumulated via the $\ell_{1}$-norm among samples, compared
with\ \eqref{eq:TED}, the robustness against outlier samples is
enhanced. Equivalently,\ \eqref{eq:nieRRSS_vector} is rewritten
in the matrix format: 
\begin{equation}
\min_{\mtbfA\in\mtbbR^{N\times N}}\left\Vert \left(\mtbfX-\mtbfX\mtbfA\right)^{T}\right\Vert _{2,1}+\gamma\left\Vert \mtbfA\right\Vert _{2,1}.\label{eq:nieRRSS_matrix}
\end{equation}
Since the objective\ \eqref{eq:nieRRSS_matrix} is convex in $\mtbfA$,
the global minimum may be found by differentiating\ \eqref{eq:nieRRSS_matrix}
and setting the derivative to zero\ \cite{ALevin_2008_PAMI_closedFormd},
resulting in a linear system%
\footnote{To avoid singular failures, we get $V_{nn}=\frac{1}{2\sqrt{\left\Vert \mtbfa^{n}\right\Vert _{2}^{2}+\epsilon}}$,
$U_{nn}=\frac{1}{2\sqrt{\left\Vert \mtbfx_{n}-\mtbfX\mtbfa_{n}\right\Vert _{2}^{2}+\epsilon}}$
$\left(\epsilon>0\right)$. Then the algorithm is to minimize the
objective of $\sum_{n}^{N}\sqrt{\left\Vert \mtbfx_{n}-\mtbfX\mtbfa_{n}\right\Vert _{2}^{2}+\epsilon}+\gamma\sum_{n}^{N}\sqrt{\left\Vert \mtbfa^{n}\right\Vert _{2}^{2}+\epsilon}$.
When $\epsilon\rightarrow0$, this objective is reduced to the objective\,\eqref{eq:nieRRSS_matrix}.%
} {\small{
\begin{equation}
\mtbfa_{n}\!=\! U_{nn}\!\left(U_{nn}\mtbfX^{T}\mtbfX+\gamma\mtbfV\right)^{-1}\mtbfX^{T}\mtbfx_{n},\ \forall n\!=\!\left\{ 1,\!2,\!\cdots\!,\! N\right\} \!,\label{eq:nieRRSS_solver}
\end{equation}
}}where $\mtbfV\in\mtbbR^{N\times N}$ is a diagonal matrix with the
$n^{\mtth}$ diagonal entry as $V_{nn}=\frac{1}{2\left\Vert \mtbfa^{n}\right\Vert _{2}}$
and $U_{nn}=\frac{1}{2\left\Vert \mtbfx_{n}-\mtbfX\mtbfa_{n}\right\Vert _{2}}$. 

It seems perfect to use\ \eqref{eq:nieRRSS_solver} to solve the
objective\ \eqref{eq:nieRRSS_matrix}, because\ \eqref{eq:nieRRSS_solver}
looks simple and the global optimum is theoretically guaranteed\ \cite{FpNie_2013_IJCAI_EarlyActiveLearning}.
Unfortunately, in terms of speed,\ \eqref{eq:nieRRSS_solver} is
usually infeasible due to the incredible computational demand in the
case of large $N$ (the number of samples). At each iteration, the
computational complexity of~\eqref{eq:nieRRSS_solver} is up to $O\left(\! N^{4}\!\right)$,
as analyzed in Remark\,\ref{remark:RRSSnie_computationalComplexity}.
According to our experiments, the time cost is  up to 2088 hours
(i.e. 87 days) for a subset selection problem of 13000 samples. 

\begin{remark}\label{remark:RRSSnie_computationalComplexity} Since
$U_{nn}\mtbfX^{T}\mtbfX\!+\!\gamma\mtbfV\!\in\!\mtbbR^{N\times N}$,
the major computational cost of\ \eqref{eq:nieRRSS_solver} focuses
on a $N\times N$  linear system. If solved by the Cholesky factorization
method, it costs $\frac{1}{3}N^{3}$ for factorization as well as
$2N^{2}$ for forward and backward substitution. This amounts to $O\left(N^{3}\right)$
in total. By now, we only solve $\mtbfa_{n}$. Once solving all the
set of $\left\{ \mtbfa_{n}\right\} _{n=1}^{N}$, the total complexity
amounts to $O\left(\! N^{4}\!\right)$ for one iteration step. \end{remark}

\section{Accelerated Robust Subset Selection (ARSS)}

Due to the huge computational costs, Nie's method is infeasible for
the case of large $N$---the computational time is up to 2088 hours
for a case of 13000 samples. Besides, Nie's model\ \eqref{eq:nieRRSS_matrix}
imposes the $\ell_{2}$-norm among features, which is prone to outliers
in features. To tackle the above two issues, we propose a more robust
model in the $\ell_{p}\left(0<p\leq1\right)$-norm. Although the resulted
objective is challenging to solve,  a speedup algorithm is proposed
to dramatically save the computational costs. For the same task of
$N=13000$, it costs our method 1.8 minutes, achieving a 68429 times
acceleration compared with the speed of Nie's method.

\subsubsection{Modeling. }

To boost the robustness against outliers in both samples and features,
we formulate the discrepancy between $\mtbfX$ and $\mtbfX\mtbfA$
via the $\ell_{p}\!\left(0\!<\! p\!<\!1\right)$-norm. There are theoretical
and empirical evidences to verify that compared with $\ell_{2}$ or
$\ell_{1}$ norms, the $\ell_{\! p}$-norm is more able to prevent
outlier elements from dominating the objective, enhancing the robustness\ \cite{fpNie_2012_ICDM_robustMatrixCompletion}.
Thus, we have the following objective
\begin{equation}
\min_{\mtbfA\in\mtbbR^{N\!\times\! N}}\mtcalO=\left\Vert \mtbfX-\mtbfX\mtbfA\right\Vert _{p}^{p}+\gamma\left\Vert \mtbfA\right\Vert _{2,1},\label{eq:ourObjectiveFunction}
\end{equation}
where $\gamma$ is a balancing parameter; $\mtbfA$ is a row sparse
matrix, used to select the most informative and representative samples.
By minimizing the energy of\ \eqref{eq:ourObjectiveFunction}, we
could capture the most essential properties of the dataset $\mtbfX$. 

After obtaining the optimal $\mtbfA$, the row indexes are sorted
by the row-sum value of the absolute $\mtbfA$ in decreasing order.
The samples specified by the top $K$ indexes are selected as exemplars.
Note that the model\ \eqref{eq:ourObjectiveFunction} could be applied
to the unsupervised feature selection problem by only transposing
the data matrix $\mtbfX$. In this case, $\mtbfA$ is a $L\times L$
row sparse matrix, used to select the most representative features.

\subsection{Accelerated Solver for the ARSS Objective in\ \eqref{eq:ourObjectiveFunction} }

Although objective\ \eqref{eq:ourObjectiveFunction} is challenging
to solve, we propose an effective and highly efficient solver. The
acceleration owes to the ALM and an equivalent derivation.

\paragraph{ALM }

The most intractable challenge of\ \eqref{eq:ourObjectiveFunction}
is that, the $\ell_{p}\left(0<p\leq1\right)$-norm is non-convex,
non-smooth and not-differentiable at the zero point. Therefore, it
is beneficial to use the Augmented Lagrangian Method (ALM)\ \cite{Nocedal_2006_book_numericalOptimization}
to solve\ \eqref{eq:ourObjectiveFunction}, resulting in several
easily tackled unconstrained subproblems. By solving them iteratively,
the solutions of subproblems could eventually converge to a minimum\ \cite{cbLi_2011_PhDthesis_CSfor3D_data,gaofengMeng_2013_ICCV_dehazing}. 

Specifically, we introduce an auxiliary variable $\mtbfE=\mtbfX-\mtbfX\mtbfA\in\mtbbR^{L\times N}$.
Thus, the objective\ \eqref{eq:ourObjectiveFunction} becomes: 
\begin{equation}
\min_{\mtbfA,\mtbfE=\mtbfX-\mtbfX\mtbfA}\left\Vert \mtbfE\right\Vert _{p}^{p}+\gamma\left\Vert \mtbfA\right\Vert _{2,1}.\label{eq:RSS_auxiliaryEquation}
\end{equation}
To deal with the equality constraint in\ \eqref{eq:RSS_auxiliaryEquation},
the most convenient method is to add a penalty, resulting in 
\begin{equation}
\min_{\mtbfA}\left\Vert \mtbfE\right\Vert _{p}^{p}+\gamma\left\Vert \mtbfA\right\Vert _{2,1}+\frac{\mu}{2}\left\Vert \mtbfE-\mtbfX+\mtbfX\mtbfA\right\Vert _{F}^{2},\label{eq:RSS_ALM_penalty}
\end{equation}
 where $\mu$ is a penalty parameter. To guarantee the equality constraint,
it requires $\mu$ approaching infinity, which may cause bad numerical
conditions. Instead, once introducing a Lagrangian multiplier, it
is no longer requiring $\mu\rightarrow\infty$\,\cite{cbLi_2011_PhDthesis_CSfor3D_data,Nocedal_2006_book_numericalOptimization}.
Thus, we rewrite\ \eqref{eq:RSS_ALM_penalty} into the standard ALM
formulation as:
\begin{equation}
\min_{\mtbfA,\mtbfE,\Lambda,\mu}\mtcalL_{A}=\left\Vert \mtbfE\right\Vert _{p}^{p}+\gamma\left\Vert \mtbfA\right\Vert _{2,1}+\frac{\mu}{2}\left\Vert \mtbfE-\mtbfX+\mtbfX\mtbfA+\frac{\Lambda}{\mu}\right\Vert _{F}^{2},\label{eq:RRS_ALM_both}
\end{equation}
where $\Lambda$ consists of $L\times N$ Lagrangian multipliers.
In the following, a highly efficient solver will be given.

\paragraph{The updating rule for  $\Lambda$}

Similar to the iterative thresholding (IT) in\ \cite{JohnWright_2009_NIPS_Rpca,fpNie_2014_ICML_newSolverForSvm},
the degree of violations of the $L\!\times\! N$ equality constraints
are used to update the Lagrangian multiplier:
\begin{equation}
\Lambda\leftarrow\Lambda+\mu\left(\mtbfE-\mtbfX+\mtbfX\mtbfA\right),\label{eq:FSS_updataForMultiplier}
\end{equation}
where $\mu$ is a monotonically increasing parameter over iteration
steps. For example, $\mu\leftarrow\rho\mu$, where $1<\rho<2$ is
a predefined parameter\,\cite{Nocedal_2006_book_numericalOptimization}.

\paragraph{Efficient solver for $\mtbfE$ }

Removing irrelevant terms with $\mtbfE$ from\ \eqref{eq:RRS_ALM_both},
we have
\begin{equation}
\min_{\mtbfE}\left\Vert \mtbfE\right\Vert _{p}^{p}+\frac{\mu}{2}\left\Vert \mtbfE-\mtbfH\right\Vert _{F}^{2},\label{eq:objFunction_E}
\end{equation}
where $\mtbfH=\mtbfX-\mtbfX\mtbfA-\frac{\Lambda}{\mu}\in\mtbbR^{L\times N}$.
According to the definition of the $\ell_{p}$-norm and the \emph{Frobenius-}norm,\ \eqref{eq:objFunction_E}
could be decoupled into $L\times N$ independent and unconstrained
subproblems. The standard form of these subproblems is 
\begin{equation}
\min_{y}f\left(y\right)=\lambda\left|y\right|^{p}+\frac{1}{2}\left(y-c\right)^{2},\label{eq:standard_Lp_problem}
\end{equation}
where $\lambda=\frac{1}{\mu}$ is a given positive parameter, $y$
is the scalar variable need to deal with, $c$ is a known scalar constant.

Zuo \emph{et al.\ }\cite{leiZhang_2013_ICCV_LpShrinkage}\emph{ }has
recently proposed a generalized iterative shrinkage algorithm to solve\ \eqref{eq:standard_Lp_problem}.
This algorithm is easy to implement and able to achieve more accurate
solutions than current methods. Thus, we use it for our problem as:
\begin{equation}
y^{*}=\max\left(\left|c\right|-\tau_{p}\left(\lambda\right),0\right)\cdot S_{p}\left(\left|c\right|;\lambda\right)\cdot\mtsign\left(c\right),\label{eq:GST_Lp_shrinkageSolver}
\end{equation}
where 
\[
\tau_{p}\left(\lambda\right)=\left[2\lambda\left(1-p\right)\right]^{\frac{1}{2-p}}+\lambda p\left[2\lambda\left(1-p\right)\right]^{\frac{p-1}{2-p}}
\]
; $S_{p}\left(\left|c\right|;\lambda\right)$ is obtained by solving
the following equation: 
\[
S_{p}\left(c;\lambda\right)-c+\lambda p\left(S_{p}\left(c;\lambda\right)\right)^{p-1}=0,
\]
which could be solved efficiently via an iterative algorithm. In this
manner,\ \eqref{eq:objFunction_E} could be sovled extremely fast.

\paragraph{Accelerated solver for $\mtbfA$ }

The main acceleration focuses on the solver of $\mtbfA$. Removing
irrelevant terms with $\mtbfA$ from\ \eqref{eq:RRS_ALM_both}, we
have
\begin{equation}
\min_{\mtbfA}\left\Vert \mtbfA\right\Vert _{2,1}+\frac{\beta}{2}\mttrace\left\{ \left(\mtbfX\mtbfA-\mtbfP\right)^{T}\left(\mtbfX\mtbfA-\mtbfP\right)\right\} ,\label{eq:objFunction_A}
\end{equation}
where $\beta=\frac{\mu}{\gamma}$ is a nonnegative parameter, $\mtbfP=\mtbfX-\mtbfE-\frac{\Lambda}{\mu}\in\mtbbR^{L\times N}$.
Since\ \eqref{eq:objFunction_A} is convex in $\mtbfA$, the optimum
could be found by differentiating\ \eqref{eq:objFunction_A} and
setting the derivative to zero. This amounts to tackling the following
linear system%
\footnote{{\small{$\mtbfV\in\mtbbR^{N\times N}$ is a positive and diagonal
matrix with the $n^{\mtth}$ diagonal entry as $V_{nn}\!=\!\frac{1}{\sqrt{\left\Vert \mtbfa^{n}\right\Vert _{2}^{2}+\epsilon}}>0$,
where $\epsilon$ is a small value to avoid singular failures\ \cite{FpNie_2013_IJCAI_EarlyActiveLearning,fyzhu_2014_JSTSP_RRLbS}. }}%
}: 
\begin{equation}
\mtbfA=\beta\left(\mtbfV+\beta\mtbfX^{T}\mtbfX\right)^{-1}\mtbfX^{T}\mtbfP.\label{eq:FSS_update_A1}
\end{equation}
As $\mtbfV+\beta\mtbfX^{T}\mtbfX\in\mtbbR^{N\times N}$,\ \eqref{eq:FSS_update_A1}
is mainly a $N\!\times\! N$ linear system. Once solved by the Cholesky
factorization, the computational complexity is highly up to $O\left(N^{3}\right)$.
This is by no means a good choice for real applications with large
$N$. In the following, an equivalent derivation of\ \eqref{eq:FSS_update_A1}
will be proposed to significantly save the computational complexity. 

\begin{theorem} \label{theorem:FSS_update_A1_A2}The $N\times N$
linear system\ \eqref{eq:FSS_update_A1} is equivalent to the following
$L\times L$ linear system: 
\begin{align}
\mtbfA & =\beta\left(\mtbfX\mtbfV^{-1}\right)^{T}\left[\mtbfI_{L}+\beta\mtbfX\left(\mtbfX\mtbfV^{-1}\right)^{T}\right]^{-1}\mtbfP,\label{eq:FSS_update_A2}
\end{align}
 where $\mtbfI_{L}$ is a $L\times L$ identity matrix. \end{theorem}
\begin{proof} Note that $\mtbfV$ is a $N\times N$ diagonal and
positive-definite matrix, the exponent of $\mtbfV$ is efficient to
achieve, i.e. $\mtbfV^{\alpha}=\left\{ V_{nn}^{\alpha}\right\} _{n=1}^{N},\forall\alpha\!\in\!\mtbbR$.
We have the following equations
\begin{align}
\mtbfA & =\beta\left(\mtbfV+\beta\mtbfX^{T}\mtbfX\right)^{-1}\mtbfX^{T}\mtbfP\nonumber \\
 & =\beta\mtbfV^{-\frac{1}{2}}\left[\mtbfV^{-\frac{1}{2}}\left(\mtbfV+\beta\mtbfX^{T}\mtbfX\right)\mtbfV^{-\frac{1}{2}}\right]^{-1}\mtbfV^{-\frac{1}{2}}\mtbfX^{T}\mtbfP\nonumber \\
 & =\beta\mtbfV^{-\frac{1}{2}}\left(\mtbfI_{N}+\beta\mtbfZ^{T}\mtbfZ\right)^{-1}\mtbfZ^{T}\mtbfP,\label{eq:proof_media_results}
\end{align}
where $\mtbfZ=\mtbfX\mtbfV^{-\frac{1}{2}}$, $\mtbfI_{N}$ is a $N\times N$
identity matrix. The following equation holds for any conditions 
\begin{equation}
\left(\mtbfI_{N}+\beta\mtbfZ^{T}\mtbfZ\right)\mtbfZ^{T}=\mtbfZ^{T}\left(\mtbfI_{L}+\beta\mtbfZ\mtbfZ^{T}\right).\label{eq:transform_equation}
\end{equation}
Multiplying\ \eqref{eq:transform_equation} with $\left(\mtbfI_{N}+\beta\mtbfZ^{T}\mtbfZ\right)\!{}^{-1}$
on the left and $\left(\mtbfI_{L}+\beta\mtbfZ\mtbfZ^{T}\right)\!{}^{-1}$
on the right of both sides of the equal-sign, we have the equation
as:
\begin{equation}
\mtbfZ^{T}\left(\mtbfI_{L}+\beta\mtbfZ\mtbfZ^{T}\right)^{-1}=\left(\mtbfI_{N}+\beta\mtbfZ^{T}\mtbfZ\right)^{-1}\mtbfZ^{T}.\label{eq:transform_equation_inv}
\end{equation}
Therefore, substituting\ \eqref{eq:transform_equation_inv} and $\mtbfZ=\mtbfX\mtbfV^{-\frac{1}{2}}$
into\ \eqref{eq:proof_media_results}, we have the simplified updating
rule as: 
\begin{equation}
\mtbfA=\beta\left(\mtbfX\mtbfV^{-1}\right)^{T}\left[\mtbfI_{L}+\beta\mtbfX\left(\mtbfX\mtbfV^{-1}\right)^{T}\right]^{-1}\mtbfP.
\end{equation}
When $N\gg L$, the most complex operation is the matrix multiplications,
not the $L\times L$ linear system.\end{proof} 

\begin{corollary} \label{corollary:ARSS_complexity_A} We have two
equivalent updating rules\ \eqref{eq:FSS_update_A1} and\ \eqref{eq:FSS_update_A2}
for the objective\ \eqref{eq:objFunction_A}. If using\ \eqref{eq:FSS_update_A1}
when $N\!\leq\! L$, and otherwise using\ \eqref{eq:FSS_update_A2}
as shown in Algorithm\ \ref{alg:FSS_update_A}, the computational
complexity of solvers for\ \textbf{\eqref{eq:objFunction_A}} is
$O\left(N{}^{2}L\right)$. Due to $N\gg L$, we have highly reduced
the complexity from $O\left(N{}^{4}\right)$ to $O\left(N{}^{2}L\right)$
compared with Nie's method. \end{corollary} 
\begin{algorithm}[t]
\caption{ for\textbf{\ \eqref{eq:objFunction_A}}: $\mtbfA^{*}=\text{ARSS}_{\mtbfA}\left(\mtbfX,\mtbfV,\mtbfP,\mtbfI_{L},\beta\right)$\textbf{
\label{alg:FSS_update_A}}}
 \textbf{Input:} $\mtbfX,\mtbfV,\mtbfP,\mtbfI_{L},\beta$

\begin{algorithmic}[1] 

\IF{ $N\leq L$ }

\STATE update $\mtbfA$ via the updating rule\ \eqref{eq:FSS_update_A1},
that is

\STATE $\mtbfA=\beta\left(\mtbfV+\beta\mtbfX^{T}\mtbfX\right)\!{}^{-1}\mtbfX^{T}\mtbfP$.

\ELSIF{ $N>L$}

\STATE update $\mtbfA$ via the updating rule\ \eqref{eq:FSS_update_A2},
that is

\STATE $\mtbfA=\mtbfB\left(\mtbfI_{L}+\mtbfX\mtbfB\right)\!{}^{-1}\mtbfP$,
where $\mtbfB=\beta\left(\mtbfX\mtbfV^{-1}\right)\!{}^{T}$. 

\ENDIF 

\end{algorithmic} 

\textbf{Output: $\mtbfA$}
\end{algorithm}
 
\begin{algorithm}[t]
\caption{for\ \eqref{eq:ourObjectiveFunction} or\ \eqref{eq:RRS_ALM_both}:
$\mtbfA^{*}=\text{ARSS}_{\text{ALM}}\left(\mtbfX,\gamma,p\right)$\textbf{\label{alg:FSS_ALM}}}
 \textbf{Input:} $\mtbfX,\gamma,p$

\begin{algorithmic}[1] 

\STATE Initialize\emph{ }$\mu\!>\!0,1\!<\!\rho\!<\!2,\epsilon\!=\!10^{-10},\mtbfA\!=\!\mtbfI_{N},\Lambda\!=\!\mtbfzero$.

\REPEAT

\STATE update $\mtbfE$ by the updating rule~\eqref{eq:GST_Lp_shrinkageSolver}.

\STATE update $\mtbfV=\left[V_{nn}\right]\in\mtbbR^{N\times N}$.

\STATE $\mtbfP=\mtbfX\!-\!\mtbfE\!-\!\frac{\Lambda}{\mu},\beta=\frac{\mu}{\gamma}$;
$\mtbfI_{L}$ is a $L\times L$ identity matrix.

\STATE $\mtbfA=\text{ARSS}_{\mtbfA}\left(\mtbfX,\mtbfV,\mtbfP,\mtbfI_{L},\beta\right)$
via \textbf{Algorithm\ }\ref{alg:FSS_update_A}.

\STATE update $\Lambda$ by the updating rule~\eqref{eq:FSS_updataForMultiplier},
$\mu\leftarrow\rho\mu$.

\UNTIL{convergence} 

\end{algorithmic} 

\textbf{Output:} $\mtbfA$ 
\end{algorithm}

The solver to update $\mtbfA$ is given in Algorithm\ \ref{alg:FSS_update_A}.
The overall solver for our model\ \eqref{eq:ourObjectiveFunction}
is summarized in Algorithm\ \ref{alg:FSS_ALM}. 

According to Theorem\ \ref{theorem:FSS_update_A1_A2} and Corollary\ \ref{corollary:ARSS_complexity_A},
the solver  for our model\ \eqref{eq:objFunction_A} is highly simplified,
as feature length is generally much smaller than data size, i.e $L\!\ll\! N$.
Similarly, Nie's method could be highly accelerated by Theorem\ \ref{theorem:RRSS_nie2our},
obtaining 500+ times speedup, as shown in Fig.\ \ref{fig:Time_vs_N}
and Table\ \ref{tab:Compare_Time_perform}.

\begin{theorem} \label{theorem:RRSS_nie2our} Nie's $N\times N$
solver\ \eqref{eq:RRSS_nie_solver}\ \cite{FpNie_2013_IJCAI_EarlyActiveLearning}
is equivalent to the following $L\times L$ linear system\ \eqref{eq:RRSS_our_solver}
\begin{align}
\mtbfa_{n}\! & =\! U_{nn}\!\left(U_{nn}\mtbfX^{T}\mtbfX+\gamma\mtbfV\right)^{-1}\mtbfX^{T}\mtbfx_{n}\label{eq:RRSS_nie_solver}\\
 & =\! U_{nn}\!\left(\mtbfX\mtbfV^{-1}\right)^{T}\!\left(U_{nn}\mtbfX\left(\mtbfX\mtbfV^{-1}\right)^{T}\!+\!\gamma\mtbfI_{L}\right)^{-1}\!\mtbfx_{n}\label{eq:RRSS_our_solver}
\end{align}
$\forall n\in\left\{ 1,2,\cdots,N\right\} $, where $\mtbfI_{L}$
is a $L\times L$ identity matrix.\end{theorem}

\begin{proof}Based on\ \eqref{eq:RRSS_nie_solver}, we have the
following equalities: 
\begin{align*}
\mtbfa_{n}\! & =\! U_{nn}\!\left(U_{nn}\mtbfX^{T}\mtbfX+\gamma\mtbfV\right)^{-1}\mtbfX^{T}\mtbfx_{n},\\
 & =\! U_{nn}\!\mtbfV^{-\frac{1}{2}}\!\left[\mtbfV^{-\frac{1}{2}}\!\left(U_{nn}\mtbfX^{T}\mtbfX\!+\!\gamma\mtbfV\right)\mtbfV^{-\frac{1}{2}}\right]^{-1}\!\mtbfV^{-\frac{1}{2}}\mtbfX^{T}\mtbfx_{n}\\
 & =\! U_{nn}\!\mtbfV^{-\frac{1}{2}}\!\left(\! U_{nn}\!\!\left(\mtbfX\mtbfV^{-\frac{1}{2}}\right)^{T}\!\!\mtbfX\mtbfV^{-\frac{1}{2}}\!+\!\gamma\mtbfI_{N}\!\right)^{-1}\!\left(\mtbfX\mtbfV^{-\frac{1}{2}}\right)^{T}\!\mtbfx_{n}\\
 & =\! U_{nn}\!\mtbfV^{-\frac{1}{2}}\!\left(\mtbfX\mtbfV^{-\frac{1}{2}}\right)^{T}\!\!\left(\! U_{nn}\mtbfX\mtbfV^{-\frac{1}{2}}\!\!\left(\mtbfX\mtbfV^{-\frac{1}{2}}\right)^{T}\!+\!\gamma\mtbfI_{L}\!\right)^{-1}\!\mtbfx_{n}\\
 & =\! U_{nn}\!\left(\mtbfX\mtbfV^{-1}\right)^{T}\!\left(U_{nn}\mtbfX\left(\mtbfX\mtbfV^{-1}\right)^{T}\!+\!\gamma\mtbfI_{L}\right)^{-1}\!\mtbfx_{n}.
\end{align*}
{\small{The derivations are equivalent; their results are equal. }}\end{proof}

\begin{figure*}[tb]
\centering{}\includegraphics[width=2.12\columnwidth]{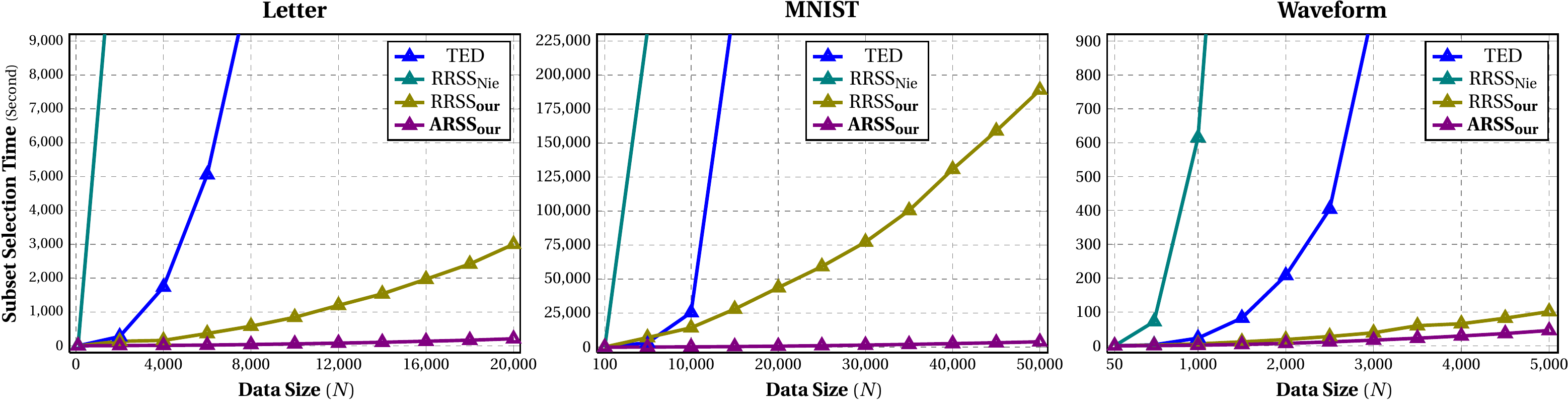}\caption{{\small{Speed }}\emph{\small{vs. }}{\small{increasing}}\emph{\small{
}}{\small{$N$ on (a) Letter, (b) MNIST and (c) Waveform. Compared
with the authorial solver TED and RRSS$_{\text{Nie}}$, our method
ARSS and RRSS$_{\text{our}}$ dramatically reduce the computational
time. The larger data size is, the larger gaps between these methods
are. Note that the selection time is not sensitive to the number of
selected samples $K$. (best viewed in color) \label{fig:Time_vs_N}}}}
\end{figure*}
\begin{table}[t]
\centering{}\caption{{\small{Statistics of ten benchmark datasets.\label{tab:description10Datasets}}} }
\includegraphics[width=1\columnwidth]{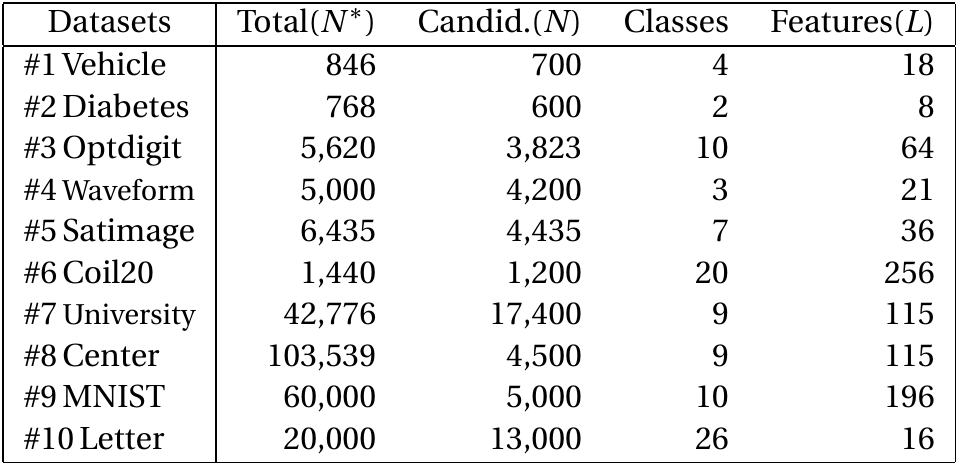}
\end{table}
\begin{corollary} \label{corollary:RRSS_our} Since feature length
is generally much smaller than data size, i.e. $L\!\ll\! N$, our
accelerated solver\ \eqref{eq:RRSS_nie_solver} for Nie's model\ \eqref{eq:nieRRSS_matrix}
is highly faster than the authorial solver\ \eqref{eq:RRSS_our_solver}.
Theoretically, we reduce the computational complexity from $O\left(N{}^{4}\right)$
to $O\left(N{}^{2}L+NL{}^{3}\right)$, while maintaining the same
solution. That is, like Nie's solver\ \eqref{eq:RRSS_nie_solver},
our speedup solver\ \eqref{eq:RRSS_our_solver} can reach the global
optimum. Extensive empirical results will verify the huge acceleration\end{corollary}

\section{Experiments}

\subsection{Experimental Settings}

In this part, the experimental settings are introduced. All experiments
are conducted on a server with 64-core Intel Xeon E7-4820 @ 2.00 GHz,
18 Mb Cache and 0.986 TB RAM, using Matlab 2012. Brief descriptions
of ten benchmark datasets are summarized in Table\ \ref{tab:description10Datasets},
where `Total$\left(N^{*}\right)$' denotes the total set of samples
in each data. Due to the high computational complexity, other methods
can only handle small datasets (while our method can handle the total
set). Thus, we randomly choose the candidate set from the total set
to reduce the sample size, i.e. $N<N^{*}$ (cf. `Total$\left(N^{*}\right)$'
and `candid.$\left(N\right)$' in Table\ \ref{tab:description10Datasets}).
The remainder (except candidate set) are used for test. Specifically,
to simulate the varying quality of samples, ten percentage of candidate
samples from each class are randomly selected and arbitrarily added
one of the following three kinds of noise: ``Gaussian'', ``Laplace''
and ``Salt \& pepper'' respectively. In a word, all experiment settings
are same and fair for all the methods.

\subsection{Speed Comparisons}

There are two parts of speed comparisons. First, how speed varies
with increasing $N$ is illustrated in Fig.\ \ref{fig:Time_vs_N}.
Then the comparison of specific speed is summarized in Table\ \ref{tab:Compare_Time_perform}.
Note that TED and RRSS$_{\text{Nie}}$ denote the authorial solver
(via authorial codes); RRSS$_{\text{our}}$ is our accelerated solver
for Nie's model via Theorem\ \ref{theorem:RRSS_nie2our}; ARSS is
the proposed method.

\paragraph{Speed \emph{vs.} increasing $N$}

To verify the great superiority of our method over the state-of-the-art
methods in speed, three experiments are conducted. The results are
illustrated in Fig.\ \ref{fig:Time_vs_N}, where there are three
sub-figures showing the speed of four methods on the benchmark datasets
of\emph{ }Letter, MNIST and Waveform respectively. As we shall see,
both selection time of TED\ \cite{kaiYu_2006_ICML_activeLearning}
and RRSS$_{\text{Nie}}$\ \cite{FpNie_2013_IJCAI_EarlyActiveLearning}
increases dramatically as $N$ increases. No surprisingly, RRSS$_{\text{Nie}}$
is incredibly time-consuming as $N$ grows---the order of curves looks
higher than quadratic. Actually, the theoretical complexity of RRSS$_{\text{Nie}}$
is highly up to $O\left(N^{4}\right)$ as analyzed in Remark\ \ref{remark:RRSSnie_computationalComplexity}.

Compared with TED and RRSS$_{\text{Nie}}$, the curve of ARSS is surprisingly
lower and highly stable against increasing $N$; there is almost no
rise of selection time over growing $N$. This is owing to the speedup
techniques of ALM and equivalent derivations. Via them, we reduce
the computational cost from $O\left(N^{4}\right)$ to $O\left(N{}^{2}L\right)$,
as analyzed in Theorem\ \ref{theorem:FSS_update_A1_A2} and Corollary\ \ref{corollary:ARSS_complexity_A}.
Moreover, with the help of Theorem\ \ref{theorem:RRSS_nie2our},
RRSS$_{\text{our}}$ is the second faster algorithm that is significantly
accelerated compared with the authorial algorithm RRSS$_{\text{Nie}}$.

\subsubsection{Speed with\emph{ }fixed $N$}

The speed of four algorithms is summarized in Table\ \ref{tab:Compare_Time_perform}a,
where each row shows the results on one dataset and the last row displays
the average results. Four conclusions can be drawn from Table\ \ref{tab:Compare_Time_perform}a.
First, ARSS is the fastest algorithm, and RRSS$_{\text{our}}$ is
the second fastest algorithm. Second, with the help of Theorem\ \ref{theorem:RRSS_nie2our},
RRSS$_{\text{our}}$ is highly faster than RRSS$_{\text{Nie}}$, averagely
obtaining a 559 times acceleration. Third, ARSS is dramatically faster
than RRSS$_{\text{Nie}}$ and TED; the results in Table\ \ref{tab:Compare_Time_perform}a
verify an average acceleration of 23275 times faster than RRSS$_{\text{Nie}}$
and 281 times faster than TED. This means that for example if it takes
RRSS$_{\text{Nie}}$ 100 years to do a subset selection task, it only
takes our method 1.6 days to address the same problem. Finally, we
apply ARSS to the whole sample set of each data. The results are displayed
in the $6^{\mtth}$ column in Table\ \ref{tab:Compare_Time_perform},
showing its capability to process very large datasets. 
\begin{table*}[t]
\begin{centering}
\caption{{\small{Performances of TED, RRSS and ARSS: }}\textbf{\scriptsize{(}}\textbf{\emph{\scriptsize{left}}}\textbf{\scriptsize{-}}\textbf{\emph{\scriptsize{a}}}\textbf{\scriptsize{)}}{\small{
speed in seconds, }}\textbf{\scriptsize{(}}\textbf{\emph{\scriptsize{right}}}\textbf{\scriptsize{-}}\textbf{\emph{\scriptsize{b}}}\textbf{\scriptsize{)}}{\small{
prediction accuracies. In terms of speed, with the help of Theorem\ \ref{theorem:RRSS_nie2our},
RRSS$_{\text{our}}$ is averagely 559+ times faster than the authorial
algorithm, i.e. RRSS$_{\text{Nie}}$; ARSS achieves surprisingly 23275+
times acceleration compared with RRSS$_{\text{Nie}}$. Due to the
more robust loss in the $\ell_{p}$-norm, the prediction accuracy
of ARSS is highly encouraging.\label{tab:Compare_Time_perform}}}}
\textbf{\includegraphics[width=2.1\columnwidth]{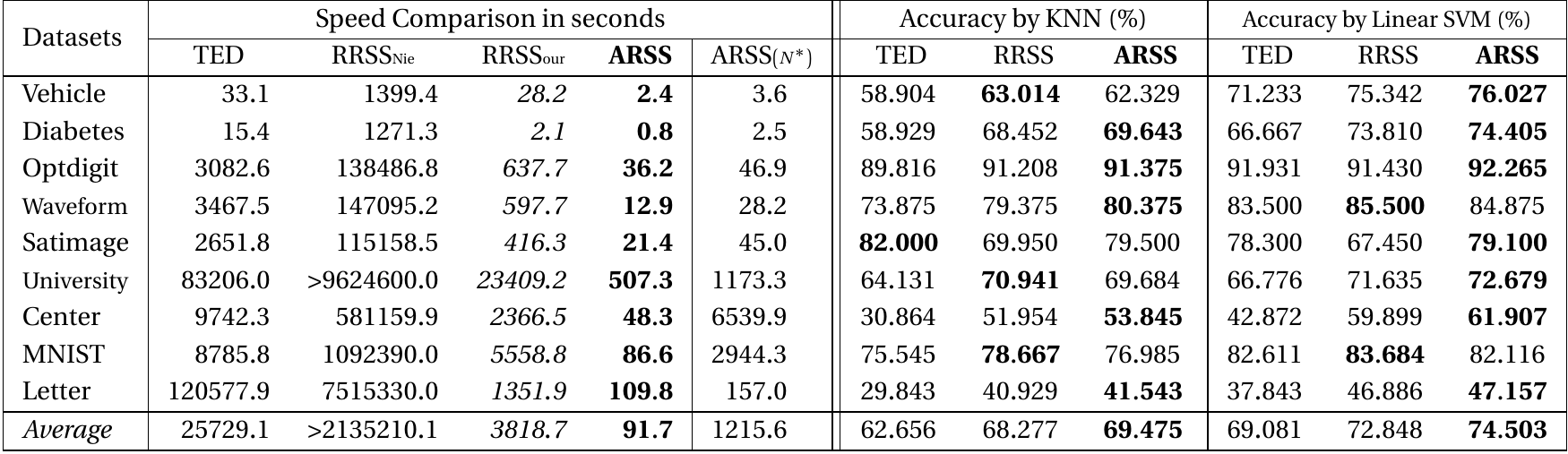}}
\par\end{centering}

\raggedright{}\emph{\footnotesize{`ARSS}}\textbf{\emph{\footnotesize{$\left(N^{*}\right)$}}}\emph{\footnotesize{'}}\textbf{\emph{\footnotesize{
}}}\emph{\footnotesize{means the task of selecting samples from the
whole dataset (with $N^{*}$ samples as shown in the $2^{\text{nd}}$column
in Table\,\ref{tab:description10Datasets}), while `TED' to `}}\textbf{\emph{\footnotesize{ARSS}}}\emph{\footnotesize{'}}\textbf{\emph{\footnotesize{
}}}\emph{\footnotesize{indicate the problem of dealing with the candidate
sample sets (with $N$ samples as shown in the $3^{\text{rd}}$ column
in Table\,\ref{tab:description10Datasets}). }}
\end{table*}
 
\begin{figure*}[t]
\centering{}\includegraphics[width=2.1\columnwidth]{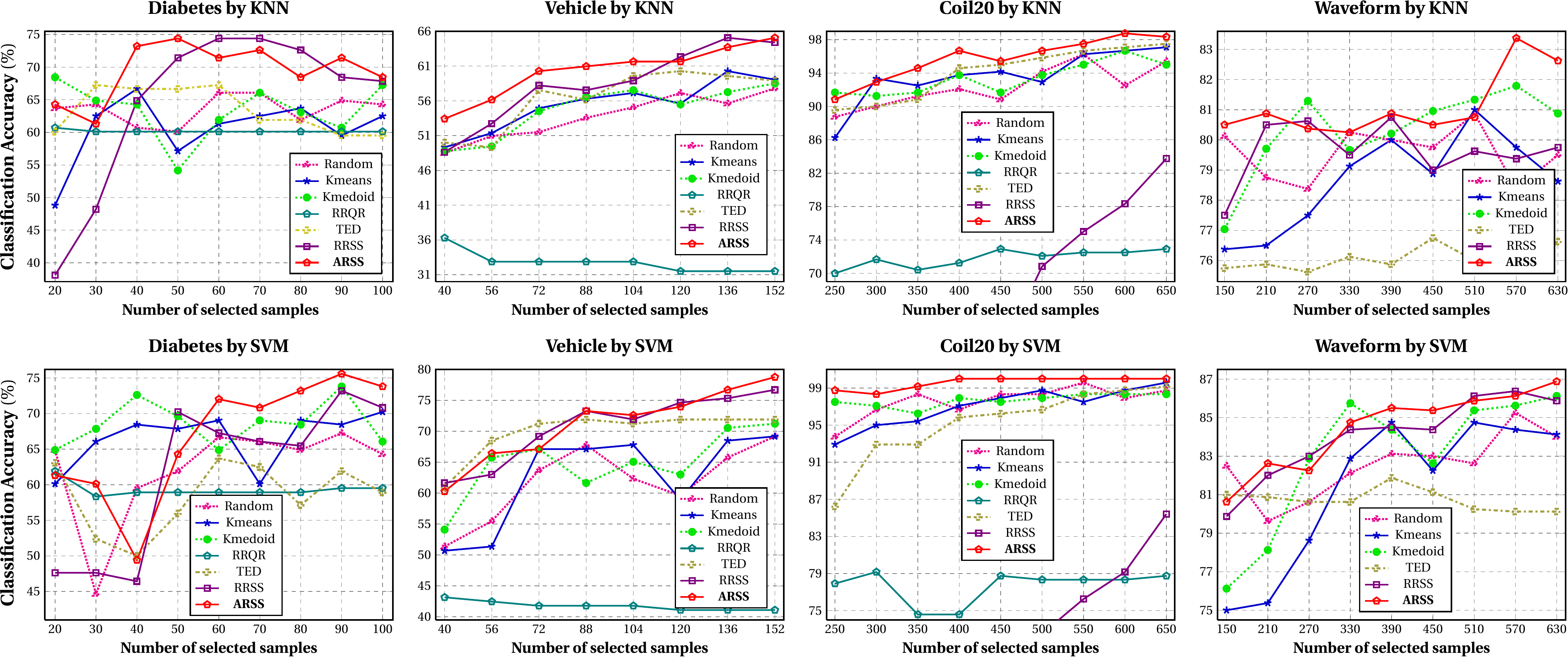}\caption{{\small{Accuracies }}\emph{\small{vs. }}{\small{increasing $K$ (the
number of selected samples). There are two rows and four columns of
subfigures: the top row shows the prediction accuracies of KNN, and
the bottom shows the results of Linear SVM; each column shows the
performances on one datasets,}}\emph{\small{ }}{\small{that is Diabetes,
Vehicle, Coil20 and Waveform respectively. Generally, ARSS (ours)
is among the best. (best viewed in color) \label{fig:accuracy_vs_K}}}}
\end{figure*}

\subsection{Prediction Accuracy}

\paragraph{Accuracy comparison}

We conduct experiments on ten benchmark datasets. For each dataset,
the top 200 representative samples are selected for training. The
prediction accuracies are reported in Table\ \ref{tab:Compare_Time_perform}b,
including the results of two popular classifiers. Three observations
can be drawn from this table. First, Linear SVM generally outperforms
KNN. Second, in general, our method performs the best; for a few cases,
our method achieves comparable results with the best performances.
Third, compared with TED, both RRSS and ARSS achieve an appreciable
advantage. The above analyses are better illustrated  in the last
row of Table\ \ref{tab:Compare_Time_perform}b. These results demonstrate
that the $\ell_{p}$ loss in our model is well suited to select exemplars
from the sample sets of various quality.

\paragraph{Prediction accuracies \emph{vs.} increasing $K$ }

To give a more detailed comparison, Fig.\ \ref{fig:accuracy_vs_K}
shows the prediction accuracies versus growing $K$ (the number of
selected samples). There are two rows and four columns of sub-figures.
The top row shows the results of KNN, and the bottom one shows results
of SVM. Each column gives the result on one dataset. As we shall see,
the prediction accuracies generally increase as $K$ increases. Such
case is consistent with the common view that more training data will
boost the prediction accuracy. For each sub-figure, ARSS is generally
among the best. This case implies that our robust objective\ \eqref{eq:ourObjectiveFunction}
via the $\ell_{p}$-norm is feasible to select subsets from the data
of varying qualities.

\section{Conclusion }

To deal with tremendous data of varying quality, we propose an accelerated
robust subset selection (ARSS) method. The $\ell_{p}$-norm is exploited
to enhance the robustness against both outlier samples and outlier
features. Although the resulted objective is complex to solve, we
propose a highly efficient solver via two techniques: ALM and equivalent
derivations. Via them, we greatly reduce the computational complexity
from $O\left(N^{4}\right)$ to $O\left(N{}^{2}L\right)$. Here feature
length $L$ is much smaller than data size $N$, i.e. $L\ll N$. Extensive
results on ten benchmark datasets verify that our method not only
runs 10,000+ times faster than the most related method, but also outperforms
state of the art methods. Moreover, we propose an accelerated solver
to highly speed up Nie's method, theoretically reducing the computational
complexity from $O\left(N^{4}\right)$ to $O\left(N{}^{2}L+NL{}^{3}\right)$.
Empirically, our accelerated solver could achieve equal results and
500+ times acceleration compared with the authorial solver.

\paragraph{Limitation. }

Our efficient algorithm build on the observation that the number of
samples is generally larger than feature length, i.e. $N\!>\! L$.
For the case of $N\!\leq\! L$, the acceleration will be inapparent.

\section*{Acknowledgements}

The authors would like to thank the editor and the reviewers for their
valuable suggestions. Besides, this work is supported by the projects
(Grant No. 61272331, 91338202, 61305049 and 61203277) of the National
Natural Science Foundation of China.

\bibliographystyle{7F__important_doingWork_myWorks_AAAI2015_aaai}
\phantomsection\addcontentsline{toc}{section}{\refname}\bibliography{5F__important_doingWork_referBib_forTIP,/home/zfy/important/doingWork/referBib_forTIP,6F__important_doingWork_referenceBib}

\end{document}